\documentclass[submission,copyright,creativecommons]{eptcs}
\providecommand{\event}{International Conference on Logic Programmin 2024 (ICLP'24)} 

\usepackage{iftex}

\ifpdf
  \usepackage{underscore}         
  \usepackage[T1]{fontenc}        
\else
  \usepackage{breakurl}           
\fi

\usepackage{amsmath}
\usepackage{graphicx}
\usepackage{multirow}

\usepackage{mathpartir}
\usepackage{enumitem}
\usepackage{xspace}
\usepackage{hyperref}

\usepackage{xcolor}

\usepackage{idp-bnf}
\usepackage{idp-math}

\usepackage{amsthm}

\newtheorem{example}{Example}
\newtheorem{proposition}{Proposition}
\newtheorem{definition}{Definition}

\newcommand{\OSL}{\textit{OSL}\xspace}

\newcommand{\Animal}{\type{A}nimal}
\newcommand{\Cat}{\type{C}at}
\newcommand{\Dog}{\type{D}og}
\newcommand{\Kind}{\type{K}ind}
\newcommand{\Sound}{\type{S}ound}

\newcommand{\Meow}{\Co{meow}}
\newcommand{\Bark}{\Co{bark}}
\newcommand{\Tom}{\Co{tom}}
\newcommand{\mS}{\Co{makingSound}}

\newcommand{\mother}{\Co{mother}}
\newcommand{\father}{\Co{father}}

\newcommand{\Lalt}{$\mid$\xspace}

\title{Order-Sorted Intensional Logic: Expressing Subtyping Polymorphism with Typing Assertions and Quantification over Concepts\thanks{This work was partially supported by the Flemish Government under the ``Onderzoeksprogramma Artificiële Intelligentie (AI) Vlaanderen''.}}

\author{\DJ or\dj e Markovi\'c \qquad\qquad Marc Denecker
\institute{KU Leuven, Department of computer science, Leuven, Belgium}
\email{\quad dorde.markovic@kuleuven.be \quad\qquad marc.denecker@kuleuven.be}
}

\def\titlerunning{Order-Sorted Intensional Logic}
\def\authorrunning{\DJ. Markovi\'c \& M. Denecker}
\begin{document}
\maketitle

\begin{abstract}
	Subtyping, also known as subtype polymorphism, is a concept extensively studied in programming language theory, delineating the substitutability relation among datatypes. This property ensures that programs designed for supertype objects remain compatible with their subtypes.
	
    In this paper, we explore the capability of order-sorted logic for utilizing these ideas in the context of Knowledge Representation. We recognize two fundamental limitations: First, the inability of this logic to address the concept  rather than the value  of non-logical symbols, and second, the lack of language constructs for constraining the type of terms. Consequently, we propose guarded order-sorted intensional logic, where guards are language constructs for annotating typing information and intensional logic provides support for quantification over concepts.
\end{abstract}

\section{Introduction}
The logic-based approach to knowledge representation (KR) dates back to the early ages of artificial intelligence. From the inception of this approach, limitations of untyped logic were identified. 
These issues led to the use of \emph{many-sorted logic}~\cite{DBLP:journals/jsyml/Wang52}, and \emph{order-sorted logic} (\OSL)~\cite{DBLP:journals/ai/BeierleHPSS92}. 
In \emph{many-sorted logic}, the domain of discourse (or universe) is partitioned into different sorts/types, all disjoint. 
The latter assumption is lifted in \emph{order-sorted logic}, and sorts/types can be organized in a hierarchy by inclusion.

When extending first-order logic with ordered sorts, the concept of \emph{subtyping polymorphism} emerges \cite[Chapter 15]{pierce2002types}.
A prime example of this concept is the modeling of characteristic behaviors among different animals. 
Dogs bark, cats meow, etc., while nearly all animals produce species-specific sounds.
In this scenario, \textit{animal} serves as the overarching type, with specific animal types acting as subtypes.
In many programming languages, one can invoke a method such as \textit{produce sound} for an animal, which \emph{dynamically dispatches} behavior based on the specific species of the animal.
Logic is characterized by model semantics, and hence, it lacks the notion of method invocation found in programming languages. Nonetheless, logical statements can draw inspiration from this concept. For instance, consider the statement: \textit{``There is an animal in my yard that is either barking or meowing''}. 
Considering that \textit{barking} and \textit{meowing} are predicates defined respectively on types \textit{dog} and \textit{cat} which are subtypes of the type \textit{animal}, expressing such statements in \OSL may easily lead to untyped formulae, as we shall see later.

In this paper, we explore the principles underlying subtyping polymorphism and highlight challenges in its representation within order-sorted first-order logic.
Additionally, we identify the two key principles essential for naturally expressing such concepts in any logic employing order-sorts.
The first principle relates to the inherent incapacity of standard \OSL to condition the subtyping relation of a term. 
For instance, given a variable $x$ ranging over type \textit{animal}, it is impossible\footnote{Here we strictly talk about the incapability to constrain the subtype of a variable, and not about possible alternative modelings that would circumvent this issue by changing the ontology.} to express \textit{``if $x$ is of type dog then $x$ is barking''}. 
Furthermore, we show that making such typing assertions implicit (annotated) is essential for subtyping. 
Whereby, annotating aims to implicitly constrain the type of a variable to the type of the argument it occurs at, given that the type of the variable is a supertype of the argument.
For example, a language can be extended such that the statement \textit{``$\langle\langle x$ is barking$\rangle\rangle$''} stands for \textit{``if $x$ is of type dog then $x$ is barking''}. 
This is possible because predicate \textit{barking} caries the typing information that an argument of type \textit{dog} is expected. 
The second principle tackles the constraint of first-order logic to solely address the values (extensions) of non-logical symbols rather than the concepts (intensions) interpreting them and the constraint to quantify over these concepts. 
Principles of intensional logic (\cite{DBLP:journals/apal/Fitting04}) become crucial in overcoming these limitations.
We demonstrate that by combining order-sorted logic with principles of intensional logic and introducing the innovative principle of implicit type conditions, we establish a novel language suitable for expressing subtyping polymorphism.

The remainder of the paper is structured as follows: (\ref{sec:preliminaries})~Order-sorted logic preliminaries; (\ref{sec:log-anaysis})~Analysis of subtyping polymorphism from the logic perspective;  (\ref{sec:guarded-osl})~Introduction of guarded and (\ref{sec:intensional-osl})~intensional \OSL; (\ref{sec:results})~Presentation of results: Guarded order-sorted intensional logic; (\ref{sec:osl-i-typeing})~Discussion on well-typedness in order-sorted intensional logic; (\ref{sec:semantics})~Semantics of the language; (\ref{sec:related-work})~Discussion of related work; (\ref{sec:conclusion})~Conclusion.

\section{Preliminaries -- Order-sorted Logic}
\label{sec:preliminaries}
This section formally defines \emph{order-sorted logic}. We start with the notion of a vocabulary.
\begin{definition}
	\label{def:osl-vocabulary}
	An \textbf{\OSL vocabulary} $\Od$ of non-logical symbols is a quadruple ($\OTd$, $\OSd$, $\OSTd$, $\Omd$) where:
	\begin{itemize}[leftmargin=2.5em]
		\item $\OTd$ is a set of \textbf{type symbols} $\Td$. Type symbols $\univ$ (universe), $\bool$ (boolean), and $\nat$ (natural numbers) are always member of $\OTd$.
        
        \item $\OSd$ is a set of \textbf{function and predicate} symbols.
		
        \item $\OSTd$ is a \textbf{subtyping} relation on $\OTd$. Type $\type{S}$ is a \textbf{direct subtype} of $\Td$ if $\type{S} \OSTd \Td$. For each type $\Td$ (except $\univ$) without  direct supertype  declaration  $\Td \OSTd \type{S}$, we implicitly assume  $\Td \OSTd \univ$. Accordingly, $\bool \OSTd \univ$ and $\nat \OSTd \univ$.
		
		\item $\Omd$ is a \textbf{type signature} associating to every symbol in $\OSd$ a word of the following format $\atypesign{T}{n}{\Td}$ (i.e., type term). 
		If $\Td = \bool$, the symbol is a predicate symbol, otherwise it is a function symbol. The sets of predicate and  function symbols are  denoted with $\OSd^p$, respectively   $\OSd^f$.
	\end{itemize}
    Type $\Td_1$ is called a \textbf{subtype} of $\Td_2$ if there is a path from $\Td_1$ to $\Td_2$ in the relation $\OSTd$.
\end{definition}

\begin{proposition}
    \label{prop:universe}
    Given vocabulary $\Od$, every type $\Td$ in $\OTd$ (except $\univ$) is subtype of $\univ$.
\end{proposition}
\begin{proof}
    Follows directly from Definition \ref{def:osl-vocabulary}.
\end{proof}

A symbol with type term $() \rightarrow \Td$ is an object (or a constant function) symbol. 
A symbol with type term $() \rightarrow \bool$ is a propositional  symbol. 

\begin{example}
	\label{eg:voc}
	The following vocabulary declares: types $\Animal$ (Animal), $\Cat$ (Cat), and $\Dog$ (Dog), where cat and dog are subtypes of animal; function $\Co{age}$ mapping animals to natural numbers; constant $\Tom$ of type cat; and two predicates $\Bark$ and $\Meow$ denoting sets of dogs and cats respectively.
	\begin{gather*}
		\VType \Animal \qquad \VType \Cat <: \Animal \qquad \VType \Dog <: \Animal \\
		\Co{age} : \Animal \rightarrow \nat \qquad \Tom : () \rightarrow \Cat \qquad \Bark : \Dog \rightarrow \bool \qquad \Meow : \Cat \rightarrow \bool
	\end{gather*}
\end{example}

The following defines \OSL terms, formulae, expressions, and sentences.
\begin{definition}
	\label{def:osl-syntax}
	Given an infinite set $X$ of variable symbols and \OSL vocabulary $\Od$, an \OSL \textbf{term} ($\tau$) and \textbf{formula} ($\phi$) over $X$ and $\Od$ are defined inductively (using BNF):
	\begin{itemize}[leftmargin=2.5em]
		\item A \textbf{term} is a variable or a function term: 
		\begin{grammar}
			<\Term> ::= <\Var> \Lalt <\Function>(<\Term>,..., <\Term>) \text{ where $x \in X, f \in \OSd^f$}
		\end{grammar}
		\item A \textbf{formula} is true or false, an atomic formula, a negation, a disjunction, or an existential quantification:
		\begin{grammar}
			<\Formula> ::= \Tr \ \Lalt \Fa \ \Lalt <\Predicate>(<\Term>,...,<\Term>) \Lalt \Neg <\Formula> \Lalt <\Formula> \Or <\Formula> \Lalt \Exists <\Var> [<\Type>] : <\Formula>\\
            \text{ where $x \in X, p \in \OSd^p, \Td \in \OTd$}
		\end{grammar}
		\item  An \OSL \textbf{expression} ($\alpha$) is either an \OSL \emph{term} or an \OSL \emph{formula}. A formula with no free variables (all variables in the formula are quantified) is a \textbf{sentence}.
	\end{itemize}
\end{definition}
Other familiar connectives, $\land$, $\Rightarrow$, $\forall$, and $\Leftrightarrow$ can be defined in the standard way as shortcuts in terms of the basic ones. Furthermore, we assume the language is equipped with the standard set of predicates and functions, i.e., $=$ for each type and standard arithmetic operations ($+$, $-$, $\times$, \dots) on natural numbers.
\begin{example}
	An example of a term: $\Co{age}(\Tom)$; formula: $\Co{age}(x) = 15$; and sentence: $\exists a[\Animal]:\Co{age}(a) = 15$.
\end{example}

Sentence $\Bark(\Tom)$ is well-formed according to Definition \ref{def:osl-syntax}.
However, the typing information shows that it is senseless as cats cannot bark.
For that reason, it is customary to define a syntactic subclass of Definition \ref{def:osl-syntax} that avoids such category clashes. These are the well-typed formulae.
\begin{definition}
	\label{def:osl-well-typed}
	Given an infinite set $X$ of variable symbols and an \OSL vocabulary $\Od$, a \textbf{typing context} $\omega$ is a set of typing annotations of the format $s:t$ where $s$ is a symbol from $X \cup \OSd$ and $t$ is a type term over $\OTd$. 
	A \textbf{typing relation} $\TypDc{\alpha}:\Td$, meaning that expression $\alpha$ is of type $\Td$ in the context $\omega$, is defined by the following inductive definition:
    \begin{mathpar}
        \inferrule*[right=\emph{(T-tr)},leftskip=1em,rightskip=-1em]{\vphantom{X_x}}{\TypDc{\Tr} : \bool}\quad
        \inferrule*[right=\emph{(T-fa)},rightskip=-1em]{\vphantom{X_x}}{\TypDc{\Fa} : \bool}\quad
        \inferrule*[Right=\emph{(T-or)},rightskip=-1em]{\TypDc{\phi : \bool} \\ \TypDc{\varphi : \bool}}{\TypDc{(\phi \lor \varphi) : \bool}}
        \\
        \inferrule*[right=\emph{(T-neg)},leftskip=-1em,rightskip=-1em]{\TypDc{\phi : \bool}}{\TypDc{\neg \phi : \bool}}
        \inferrule*[right=\emph{(T-ex)},rightskip=-1em]{\Typ{\omega \cup \{x : \Td\}}{\phi : \bool}}{\TypDc{(\exists x [\Td] : \phi) : \bool}}
        \inferrule*[right=\emph{(T-sub)}]{\TypDc{t : \mathbb{S}} \\ \mathbb{S} \OSTd \Td}{\TypDc{t : \Td}}
        \\
        \inferrule*[right=\emph{(T-var)},leftskip=1em,rightskip=-1em]{x:\Td \in \Tc}{\TypDc{x} : \Td}\quad
        \inferrule*[Right=\emph{(T-app)},leftskip=1em]{s : \Td_1 \times \dots \times \Td_n \rightarrow \Td \in \omega \\ \TypDc{t_1}:\Td_1 \ \dots \ \TypDc{t_n}:\Td_n}{\TypDc{s(t_1,\dots,t_n)} : \Td}
    \end{mathpar}
 	An \OSL expression $\alpha$ over \OSL vocabulary $\Od$ with free variables $x_1, \dots, x_n$ is \textbf{well-typed} iff there are types $\Td_1,\dots, \Td_n, \Td \in \OTd$ such that $\omega = \{s:t \mid ( s \in \OSd \text{ and } \Omd(s) = t ) \text{ or } (s \text{ is } x_i \text{ and } t \text{ is } \Td_i \text{ for } i \in 1\dots n)\}$ and $\TypDc{\alpha}:\Td$. An \OSL sentence $\psi$ is well-typed iff $\TypDc{\psi}:\bool$ where $\omega = \{s:t \mid s \in \OSd \text{ and } \Omd(s) = t \}$ (as $\psi$ has no free variables).
\end{definition}
The rules in this definition (a.k.a. typing judgments) define the type of an expression in a context (below the line) given that certain conditions (above the line) are satisfied.
The specific rules are existential quantification (T-\textsc{ex}) which introduces new typing annotation to the context\footnote{Alternatively, one could say that this rule projects away typing information, depending on whether the rule is interpreted top-down or bottom-up.}, and subtyping rule (T-\textsc{sub}) which expresses that the term of type $\mathbb{S}$ can be seen as of type $\Td$ if it holds that $\Td$ is a supertype of $\mathbb{S}$. 
For an \OSL sentence to be well-typed, the context initially has to correspond to the type signature of the function and predicate symbols from the vocabulary ($\Omd$).
Given the vocabulary from Example \ref{eg:voc}, the formula $\Bark(\Tom)$ is ill-typed (i.e., not well-typed) because predicate $\Bark$ expects argument of type $\Dog$ and $\Tom$ is of type $\Cat$. In general and informally, a formula is well-typed if the type of each expression occurring as an argument to a function/predicate symbol is a subtype or of the same type as the type of that argument.

\section{Analysis of subtyping polymorphism}
\label{sec:log-anaysis}
As previously noted, the statement $\Bark(\Tom)$ is considered unacceptable (ill-typed) due to the category clash it contains. 
Specifically, barking does not apply to cats.
One might argue that such statement could be accepted if always interpreted as false, thereby justifying its meaning as ``Tom is a dog and $\Bark(\Tom)$''. 
Since Tom is not a dog, the statement is false.
But what then is the meaning of $\neg \Bark(\Tom)$?
If it is interpreted as ``Tom is a dog and $\neg \Bark(\Tom)$'', then this formula is false, violating the law of excluded middle.
An alternative interpretation is ``$\neg($Tom is a dog and $\Bark(\Tom))$'', in which case the formula is true, which seems to be a more reasonable choice in this case.
However, notice that this statement is equivalent to ``If Tom is a dog then $\neg\Bark(\Tom)$''.

This brings us to an alternative interpretation of ill-typed formulae. 
One could argue that the initial formula $\Bark(\Tom)$ should be interpreted as ``If Tom is a dog then $\Bark(\Tom)$''. 
Consequently, it is justified to assert that statement $\Bark(\Tom)$ carries ambiguity, and hence can be considered as potentially dangerous, and therefore should be rejected (corresponding to a well-typed criterion).
However, we argue that extending \OSL language to support explicitly disambiguated forms of these ill-typed formulae is beneficial. We demonstrate this in the remainder of the section.

Consider the definition of the predicate $\mS : \Animal \rightarrow \bool$ representing the set of all animals producing their specific sound. 
In the running example cats and dogs. This can be formalized in \OSL as:
\begin{equation}
	\label{eq:osl-guarding}
	\begin{gathered}
		\forall a[\Animal] : \mS(a) \Leftrightarrow 
        \left(\begin{array}{cc}
            (\exists c[\Cat] : a \stackrel{_{\Animal-\Cat}}{=\joinrel=} c \land \Meow(c)) \lor\\
            (\exists d[\Dog] : a \stackrel{_{\Animal-\Dog}}{=\joinrel=} d \land \Bark(d))
        \end{array}\right)
	\end{gathered}
\end{equation}
Note that equalities $\stackrel{_{\Animal-\Cat}}{=\joinrel=}$ and $\stackrel{_{\Animal-\Dog}}{=\joinrel=}$ are necessary since they operate on different types. Returning to the main point, in this example, it would be beneficial to constrain the type of variable $a$ which ranges over type $\Animal$ in the following way.
\begin{equation}
\label{eq:conjunction-guarding}
\begin{gathered}
	\forall a[\Animal] : \mS(a) \Leftrightarrow \bigl((\Cat(a) \land \Meow(a)) \lor (\Dog(a) \land \Bark(a))\bigr)
\end{gathered}
\end{equation}
Similarly, the statement ``all animals produce their specific sound'' could be expressed as:
\begin{equation}
	\label{eq:implication-guarding}
    \forall a[\Animal] : \bigl((\Cat(a) \Rightarrow \Meow(a)) \land (\Dog(a) \Rightarrow \Bark(a))\bigr)
\end{equation}
However, these do not constitute \OSL formulae as types are used as predicates and variable $a$ of type $\Animal$ remains an argument of predicates $\Meow$ and $\Bark$.
Notice that there is room for improvement in the statements above. 
Specifically, capability to talk about ``sounds specific'' for an animal would enhance the expressivity of the language. 

Accordingly, the first goal of this paper is to extend order-sorted logic by introducing new language constructs (\emph{guards}) as motivated in this section. 
The next step is to make these guards implicit, so it is possible to express $\Cat(a) \Rightarrow \Meow(a)$ as $\Gimp{\Meow(a)}$.
Finally, to be able to talk about ``sounds specific'' for an animal the language needs to be extended with the intensional logic.
These are presented in Section~\ref{sec:guarded-osl} and~\ref{sec:intensional-osl}.

\section{Guarded order-sorted logic}
\label{sec:guarded-osl}
The extension of \OSL with the concept of guarding terms by typing assertions is characterized in the following definition.

\begin{definition}
	\label{def:guarder-osl-vocabulary}
    Definition \ref{def:osl-vocabulary} of an \OSL vocabulary $\Od$ is extended with the following rule: if $\Td$ is a type symbols in $\OTd$, then $\Tp{\Td} \in \OSd$ and $\Omd(\Tp{\Td}) = \univ \rightarrow \bool$.
    
	Definition \ref{def:osl-well-typed} of an \OSL typing relation, is extended with the two new rules, namely conjunction guarding (G-c) and implication guarding (G-i): 
    \begin{mathpar}
        \inferrule*[right=\emph{(G-c)}]{\Tc \mathop{\Ty}_{i=1}^{n} t_i : \univ \quad \Typ{\omega \mathop{\cup}_{i=1}^{n} \{t_i : \Td_i\}}{\phi : \bool}}{\TypDc{(\Tp{\Td_1}(t_1) \land \dots \land \Tp{\Td_n}(t_n) \land \phi) : \bool}} \ \qquad
        \inferrule*[right=\emph{(G-i)}]{\Tc \mathop{\Ty}_{i=1}^{n} t_i : \univ \quad \Typ{\omega \mathop{\cup}_{i=1}^{n} \{t_i : \Td_i\}}{\phi : \bool}}{\TypDc{(\Tp{\Td_1}(t_1) \land \dots \land \Tp{\Td_n}(t_n) \Rightarrow \phi) : \bool}}
    \end{mathpar}
\end{definition}

\begin{example}
	\label{eg:d-guard}
	In the guarded \OSL, statement \textit{``There is an animal (that is a cat) meowing!''} can be expressed as: $\exists a[\Animal] : \Tp{\Cat}(a) \land \Meow(a)$.
    Towards making the judgment that this formula is well-typed (i.e., of type $\bool$), let the context $\Tc$ correspond to the typing signature of vocabulary from Example~\ref{eg:voc}: 
    \begin{equation*}
    	\Tc = \{\Tp{\Cat} : \univ \rightarrow \bool;\ \Tp{\Dog} : \univ \rightarrow \bool;\ \Co{age} : \univ \rightarrow \nat; \dots\ \Meow : \Cat \rightarrow \bool\}
    \end{equation*}
    For compact representation of derivation we use the following abbreviations:
    \begin{equation*}
    	\Tc' = \Tc \cup \{a:\Animal\} \qquad \qquad \Tc'' = \Tc' \cup \{a:\Cat\}
    \end{equation*}
    The following derivation provides the judgment that this formula is well-typed:
    \begin{mathpar}
    	\inferrule* [Right=\emph{T-ex}]
    	{\inferrule* [Right=\emph{G-c}]
    		{
    		\inferrule* [Left=\emph{T-sub}]
    			{
    				\inferrule*{\checkmark_1}{\Tc' \Ty a : \Animal} \quad
    				\inferrule*{\checkmark_2}{\Animal \OSTd \univ}
    			}
    			{\Tc' \Ty a : \univ} \qquad
    		\inferrule* [Right=\emph{T-app}]
    			{
       			  \inferrule*{\checkmark_3}{\Meow : \Cat \rightarrow \bool \in \Tc''} \quad
    			    \inferrule*{\checkmark_4}{a : \Cat \in \Tc''}
    			}
    			{\Tc'' \Ty \Meow(a) : \bool}
    		}
    		{\Tc' \Ty \Tp{\Cat}(a) \land \Meow(a) : \bool}
    	}
    	{\TypDc{\exists a[\Animal] : \Tp{\Cat}(a) \land \Meow(a) : \bool}}
    \end{mathpar}
    The justification for each of the final premises ($\checkmark$) is:
    \begin{description}
        \item[$\checkmark_1$] $a : \Animal \in \Tc'$ since $\Tc' = \Tc \cup \{a:\Animal\}$.
    	\item[$\checkmark_2$] Since $\Animal$ has no supertype in $\Od$, it follows that $\Animal \OSTd \univ$ (Definition~\ref{def:osl-vocabulary}).
    	\item[$\checkmark_3$] $\Meow : \Cat \rightarrow \bool \in \Tc''$ since it is in $\Tc$ and $\Tc'' = \Tc \cup \{a:\Animal\} \cup \{a:\Cat\}$.
    	\item[$\checkmark_4$] $a : \Cat \in \Tc''$ since $\Tc'' = \Tc' \cup \{a:\Cat\}$ (due to Definition~\ref{def:guarder-osl-vocabulary}, rule \emph{(G-\textsc{c})}).
    \end{description}
\end{example}

Further, it is possible to make these typing annotations implicit by introducing new language constructs.
\begin{definition}
	\label{def:implicit-guards}
	Let $\psi$ be an \OSL formula, $\omega$ a typing context, and $\{t_1, \dots, t_n\}$ terms in $\psi$ (over \OSL vocabulary $\Od$) such that: (1) $\TypDc{t_i : \Td_i}$; (2) $t_i$ occurs in $\psi$ as an argument of predicate/function that expects argument of type $\mathbb{S}_i$; (3) $\mathbb{S}_i <:_{\Od} \Td_i$; then:
	\begin{equation*}
		\begin{split}
			\Gand{\psi} & \quad \text{ stands for } \quad \Tp{\type{S}_1}(t_1) \land \dots \land \Tp{\type{S}_n}(t_n) \land \psi\\
			\Gimp{\psi} & \quad \text{ stands for } \quad \Tp{\type{S}_1}(t_1) \land \dots \land \Tp{\type{S}_n}(t_n) \Rightarrow \psi
		\end{split}
	\end{equation*}
\end{definition}

\begin{example}
    Employing implicit guarding, Example \ref{eg:d-guard} becomes: $\exists a[\Animal] : \Gand{\Meow(a)}$.
\end{example}

\section{Order-sorted intensional logic}
\label{sec:intensional-osl}
The main concern of intensional logic is the difference between a concept (or intension)  and,  its value (or extension) in a state of affairs. A prototypical example is the ``morning star'' and ``evening star'', which represent distinct concepts (respectively, the star  in the east before sunrise, and  the star in the west after sunset),   while denoting the same object in the actual state of affairs  (the planet Venus). In the computational intensional logic of \cite{DBLP:journals/corr/abs-2202-00898}, intensions of predicates are first class objects that can be quantified over and stored in other predicates. 
For example, given a predicate $\Co{humanDisease}$ containing a set of intensions of unary predicates over humans (e.g., $\Co{flu}, \Co{measels}, \dots$) and type $\con$ representing all concepts, one can define $\Co{healtyHuman}$ as:
\begin{equation*}
\forall h [\type{H}uman]: (\Co{healthyHuman}(h) \Leftrightarrow \neg \exists c [\con]: \Co{humanDisease}(c) \land \$(c)(h))
\end{equation*}
Here $\$(c)$ is the value of the intensional object $c$. 
A similar approach can be applied to improve the formula~(\ref{eq:conjunction-guarding}) from Section~\ref{sec:log-anaysis}; here $\Co{sound}$ is a unary predicate over animal sound intensions (in the running example $\Meow$ and $\Bark$):
\begin{equation}
\label{eq:informal-intensional}
\forall a[\Animal]: \mS(a) \Leftrightarrow \exists c[\con] : sound(c) \land \$(c)(a)    
\end{equation}
However, this formula has a typing issue since sound concepts (variable $c$) can not be applied to an arbitrary animal (variable $a$), which is done by $\$(c)(a)$.
This issue will be addressed after we formally introduce ordered-sorted intensional logic.
First, a new built-in type $\con$  representing the set of concepts of the vocabulary is added to the \OSL vocabulary. This type represents the collection of all symbols (types, functions, and predicates) within the vocabulary. 
The concept associated with a symbol $s$ is denoted by $\simV{s}$ and can be accessed with the reference operator $`(s)$. 
The dual dereference operator $\$(\simV{s})$ is a unary higher-order function that, given a concept $\simV{s}$, returns the function or predicate associated with the symbol $s$.
Therefore, $\$(\simV{s})$ is always followed by another bracket containing a tuple of terms that are applied to the resulting function or predicate. Accordingly, these terms should match the type and arity of the symbol.
Formally:
\begin{definition}
	\label{def:intensional}
	The order-sorted intensional logic is defined by the following extensions:
    \begin{enumerate}[leftmargin=2.5em]
        \item An \OSL vocabulary $\Od$ contains the build-in type $\con$ (concepts).
        \item Type $\con$ denotes the set of all concepts in the vocabulary $\{\simV{s} \mid s \in \OSd \cup \OTd\}$.
        \item Given an \OSL vocabulary $\Od$, for $s \in \OSd \cup \OTd$, $`(s)$ is a term of type $\con$.
        \item If term $c$ is of type $\con$ then $\$(c)(\bar{t})$ is an \OSL expression, where $\bar{t}$ is a tuple of terms.
    \end{enumerate}
\end{definition}
\begin{example}
	\label{eg:voc-2}
	In the running example, type $\con$ denotes the set $\{\simV{\bool},$ $\simV{\nat},$ $\simV{\Animal},$ $\simV{\Cat},$ $\simV{\Dog},$ $\simV{\Co{age}},$ $\simV{\Tom},$ $\simV{\Bark},$ $\simV{\Meow}\}$. 
    Type $\Sound$ (sounds) of animals can be declared as:
    \begin{equation*}
        \VType \Sound <: \con := \{`(\Meow), `(\Bark)\}
	\end{equation*}
	Notation $\Td := \{\dots\}$ declares extension of type $\Td$.
	Term $`(\Meow)$ denotes the concept $\simV{\Meow}$.
	An example of a formula is: $\Meow(\$(`\Tom)())$, which is the same as: $\Meow(\Tom)$.
\end{example}

Consider the statement $\Bark(\$(`\Tom)())$. 
It is a well-formed formula according to the Definition~\ref{def:intensional}. It expresses that the extension of the intension of $\Tom$ is barking, which is a complex way to say that $\Tom$ is barking, i.e., $\Bark(\Tom)$. The utility of this sort of expression will become apparent only in a few paragraphs. However, here is important to notice that this statement is not well-typed, as $\Tom$ is of type $\Cat$ and $\Bark$ is a predicate expecting an argument of type $\Dog$. Furthermore, the Definition \ref{def:osl-well-typed} (well-typed formulae), does not account for intensional logic. The criteria for well-formedness and well-typedness of a formula becomes challenging in intensional logic. This is because these properties become dependent on the extensions of types and other symbols (for more details see Section~\ref{sec:osl-i-typeing}). For this paper, it suffices to reinstate these criteria by verifying the \emph{grounded} version of a formula. Grounding a variable in a formula involves substituting it with individuals from the domain of its type.  Additionally, intensional terms of the form $`(s)$ are grounded to $\simV{s}$ and intensional application $\$(\simV{s})(\bar{t})$ to $s(\bar{t})$ (here $s$ is a symbol form a vocabulary). We demonstrate this idea on the following example.
\begin{example}
	\label{eg:grounding}
    Consider the following formalization (using the type $\Sound$) of the statement from formula~(\ref{eq:informal-intensional}): ``An animal is making sound iff there is a sound it is producing''.
	\begin{gather*}
		\forall a[\Animal]: \mS(a) \Leftrightarrow \exists s[\Sound] : \$(s)(a)
	\end{gather*}
	Grounding quantification over $\Sound$ results in:
	\begin{gather*}
		\forall a[\Animal]: \mS(a) \Leftrightarrow \$(\simV{\Meow})(a) \lor \$(\simV{\Bark})(a).
	\end{gather*}
	Eliminating intensional terms results in:
	\begin{gather*}
		\forall a[\Animal]: \mS(a) \Leftrightarrow
		\Meow(a) \lor \Bark(a).
	\end{gather*}
	The grounded formula is not well-typed as variable $a$ of type $\Animal$ occurs as an argument of type $\Cat$ and $\Dog$. Therefore we conclude that the initial formula is ill-typed.
\end{example}
Restoring the well-typedness of this formula necessitates guarding of the expression $\$(s)(a)$.
Guarding this expression is challenging due to its intensional nature (i.e., variable $s$ ranges over sounds).
Consequently, the expression $\$(s)(a)$ has to be guarded depending on the value of $s$. 
This can be done by establishing a relation between \emph{animal kinds} and their \emph{specific sounds}.
One common approach is introducing an auxiliary intensional type of \emph{animal kinds} and intensional function mapping these \emph{kinds to their sounds}.
Type $\Kind$ (consisting of concepts $\simV{\Tp{\Cat}}$ and $\simV{\Tp{\Dog}}$) and function $\Co{soundOfKind}$ are declared as:
\begin{equation*}
    \VType \Kind <: \con := \{`(\Tp{\Cat}), `(\Tp{\Dog})\} \qquad\qquad \Co{soundOfKind} : \Kind \rightarrow \Sound
\end{equation*}
The following axioms define the mapping (the extension) of the function $\Co{soundOfKind}$: 
\begin{equation*}
    \Co{soundOfKind}(`(\Tp{\Cat})) = `(\Meow) \qquad\qquad \Co{soundOfKind}(`(\Tp{\Dog})) = `(\Bark)
\end{equation*}
Finally, the formula is guarded as:
\begin{equation}
	\label{eq:intensional-guarding}
	\forall a[\Animal]: \mS(a) \Leftrightarrow \exists k[\Kind] : \$(k)(a) \land \$(\Co{soundOfKind}(k))(a).
\end{equation}
The grounded version of this formula corresponds to the formula (\ref{eq:conjunction-guarding}), which is well-typed.

\section{Guarded order-sorted intensional logic}
\label{sec:results}

Formula (\ref{eq:intensional-guarding}) enhances the original statement (\ref{eq:conjunction-guarding}) by employing intensional constructs for guarding it.
However, achieving this required the introduction of a helper function relating kinds to their sounds, despite this information being present in the type of predicates $\Meow$ and $\Bark$. 
We address this issue by integrating guards (Section~\ref{sec:guarded-osl}) and intensional logic (Section~\ref{sec:intensional-osl}).
First, we demonstrate it on the running example.

\begin{example}
	\label{eg:compact-making-sound}
	Recall the formula (\ref{eq:conjunction-guarding}):
	\begin{gather*}
		\forall a[\Animal]: \mS(a) \Leftrightarrow (\Tp{\Cat}(a) \land \Meow(a)) \lor (\Tp{\Dog}(a) \land \Bark(a)).
	\end{gather*}
	Employing implicit guarding, the same can be expressed as:
	\begin{gather*}
		\forall a[\Animal]: \mS(a) \Leftrightarrow \Gand{\Meow(a)} \lor \Gand{\Bark(a)}.
	\end{gather*}
	Introducing quantification over $\Sound$ (sounds) results in:
	\begin{gather*}
		\forall a[\Animal]: \mS(a) \Leftrightarrow \exists s[\Sound] : \Gand{\$(s)(a)}.
	\end{gather*}
\end{example}

In this example, we began with the explicitly guarded formula and condensed it into a compact version using implicit guarding and quantification over concepts. Consequently, the resulting statement is well-typed.  Notably, variable $a$ is implicitly constrained to the appropriate type based on the predicate to which it is applied. 
This reflects the main goal of the paper, which is incorporating subtyping polymorphism in order-sorted logic.

\begin{example}
	The same methodology applies to formula (\ref{eq:implication-guarding}):
	\begin{gather*}
		\forall a[\Animal] : \bigl((\Tp{\Cat}(a) \Rightarrow \Meow(a)) \land 
		(\Tp{\Dog}(a) \Rightarrow \Bark(a))\bigr)	
	\end{gather*}
	Using implicit guarding on this formula we obtain: 
	$\forall a[\Animal] : \Gimp{\Meow(a)} \land \Gimp{\Bark(a)}$,
	and with quantifying over $\Sound$:
	$\forall a[\Animal] : \forall s[\Sound] : \Gimp{\$(s)(a)}$.
\end{example}

Previous examples demonstrate principles for expressing properties of objects depending on their type using \emph{guarded order-sorted intensional logic}.
The following proposition generalizes the modeling principles discussed so far.
\begin{proposition}
	\label{prop:intensional-guarding}
	Given \OSL vocabulary $\Od$:
	\begin{itemize}[leftmargin=2.5em]
		\item Let $p_1, \dots p_m$ be n-ary predicate symbols in $\Od$
		\item Let these symbols have type signature in $\Od$ as:
		\begin{gather*}
			\Omd(p_1) = \Td_{11} \times \dots \times \Td_{1n} \rightarrow \bool
            \qquad\dots\qquad
			\Omd(p_m) = \Td_{m1} \times \dots \times \Td_{mn} \rightarrow \bool
		\end{gather*}
		\item Let $\type{S}_1, \dots, \type{S}_n$ be types in $\Od$ such that:
		\begin{gather*}
			\Td_{11} <: \type{S}_1 \ \dots\ \Td_{m1} <: \type{S}_1 \qquad \dots \qquad \Td_{1n} <: \type{S}_n \ \dots\ \Td_{mn} <: \type{S}_n
		\end{gather*}
		\item Let $\type{P}$ be a type in $\Od$ such: $\type{P} <: \con := \{`(p_1), \dots, `(p_m)\}$.
		\item Let $p$ be a term of type $\type{P}$, and $t_i$ term of type $\type{S}_i$.
	\end{itemize}
	Then the following two expressions are well-typed:
	\begin{gather*}
		\Gand{\$(p)(t_1, \dots, t_n)} \qquad \Gimp{\$(p)(t_1, \dots, t_n)}
	\end{gather*}
\end{proposition}
\begin{proof}
	Term $p$ denotes some $\simV{p_k}$ from $\type{P}$ (recall that term $`(p_k)$ stands for value $\simV{p_k}$).
	The symbol $p_k$ is associated with a type term $\Td_{k1} \times \dots \times \Td_{kn} \rightarrow \bool$. 
	Accordingly, $\Gand{\$(p)(t_1, \dots, t_n)}$ stands for: $\Tp{\Td_{k1}}(t_1) \land \dots \land \Tp{\Td_{kn}}(t_n) \land p_k(t_1, \dots, t_n)$. Each of the terms $t_i$ is of type $\type{S}_i$ and hence also of type $\univ$ (follows from Proposition~\ref{prop:universe} and Definition~\ref{def:osl-well-typed} rule (T-\textsc{sub})), so each atom $\Tp{\Td_{ki}}(t_i)$ is well-typed (Definition~\ref{def:guarder-osl-vocabulary}). 
	Finally, according to Definition \ref{def:guarder-osl-vocabulary} rule (G-\textsc{c}), atom $p_k(t_1, \dots, t_n)$ is well-typed as the type of each $t_i$ is $\Td_{ki}$. 
	The proof for $\Gimp{\$(p)(t_1, \dots, t_n)}$ is similar.
\end{proof}
Patterns characterized in this proposition are essential for expressing logical statements containing subtyping polymorphism as demonstrated in previous examples.

An important observation is that the presented approach enables the compact formalization of statements involving subtyping. For instance, formula (\ref{eq:osl-guarding}) defining the predicate $makingSound$ in native \OSL, yields a formula whose length scales linearly with the number of animal kinds; i.e., the addition of another animal kind (e.g., mouse) would result in the formula growing in size. However, the logic presented in this paper is capable of expressing the same statements with formulae of constant length by utilizing the concepts introduced in Proposition \ref{prop:intensional-guarding}, as demonstrated in the examples above. Formally:
\begin{proposition}
	\label{prop:formula-length}
	Given the same environment as in Proposition \ref{prop:intensional-guarding}, the following formulae cannot be expressed in an \OSL formula with a length independent of the size of $\type{P}$:
	\begin{gather*}
		\exists s [\type{P}]:\Gand{\$(s)(t_1, \dots, t_n)} \qquad \forall s [\type{P}]: \Gimp{\$(s)(t_1, \dots, t_n)}
	\end{gather*}
\end{proposition}
\begin{proof}
	 Rewriting these formulae into \OSL requires the mentioning of all symbols in $\type{P}$.
\end{proof}

\section{Well-typedness in order-sorted intensional logic}
\label{sec:osl-i-typeing}

We argued in Section~\ref{sec:intensional-osl} that the well-typedness of formulae with intensional language constructs is not trivial. In this section, we elaborate on these issues and propose the foundations for the typing system suitable for the new language. 

Recall the methodology employed in formula~(\ref{eq:intensional-guarding}) to guard the formula in Example~\ref{eg:grounding}. 
We introduced a function $\Co{soundOfKind}$ to establish the connection between animal kinds and their specific sounds. 
It is important to note that the well-typedness of formula~(\ref{eq:intensional-guarding}) depends on the correct mapping of animal kinds to sounds by this function.
For example, if the function incorrectly maps $\simV{\Cat}$ to $\simV{\Bark}$, the formula~(\ref{eq:intensional-guarding}) would be ill-typed.
This underscores the dependence of well-typedness on the extensions (values) of types and functions.
However, the typing system from Definition~\ref{def:osl-well-typed} cannot account for such dependencies, as the type of function $\Co{soundOfKind}$ does not provide sufficient information. 

The first step towards a richer type system is the introduction of typing annotations that would clarify the typing of a concept. 
This idea is presented in \cite[Section 4]{DBLP:journals/corr/abs-2202-00898}.
For example, when quantifying over concepts, one has to provide information about the type of these concepts. 
\begin{equation*}
	\forall s \in \con[\Animal \rightarrow \bool] : \psi
\end{equation*}
In this statement variable $s$ ranges over concepts from the vocabulary which are of type $\Animal \rightarrow \bool$. 
In Example~\ref{eg:voc} these are $\Cat$, $\Dog$, $\Bark$, $\Meow$.
Similar information can be provided in the declaration of subtypes of concepts. 
For example, declaring a new type ``kind of animals'' (earlier introduced for fixing Example~\ref{eg:grounding}) requires annotating that each element of this type is a predicate over the ``animal'' type. 
Hence, the type ``kind of animals'' is a subtype of predicate concepts that are of type $\Animal \rightarrow \bool$.
\begin{equation*}
	\VType \Kind <: \con[\Animal \rightarrow \bool] := \{`(\Tp{\Cat}), `(\Tp{\Dog})\}
\end{equation*}

However, this approach fails to fully support guarded \OSL.
For example, no type can substitute $(?)$ in the following declaration of type $\Sound$ from Example~\ref{eg:voc-2}. This is because $\Meow$ and $\Bark$ are predicates over different types, $\Dog$ and $\Cat$ respectively.
\begin{equation*}
    \VType \Sound <: \con[ (?) \rightarrow \bool] := \{`(\Meow), `(\Bark)\}
\end{equation*}
Furthermore, essential for the well-typedness of formula~(\ref{eq:intensional-guarding}) is the fact that objects of type $\Kind$ are type predicates, and hence can serve for guarding. 
To make this distinction, two new types can be added: $\con^T$ to represent type concepts and $\con^S$ for function/predicate concepts.
We propose the following syntax for declaring $\Kind$ and $\Sound$:
\begin{gather*}
    \VType \Kind <: \con^T[\Animal] := \{`(\Cat), `(\Dog)\}\\
    \VType \Sound[t:\Kind] <: \con^S[ t \rightarrow \bool] := \{`(\Meow), `(\Bark)\}
\end{gather*}
Here, $\con^T[\Animal]$ stands for type concepts that are subtypes of type $\Animal$.
Notation $\Sound[t:\Kind] <: \con^S[ t \rightarrow \bool]$ expresses that type $\Sound$ is subtype of predicate concepts of type $t \rightarrow \bool$ where $t$ is of type $\Kind$ (making $\Sound$ dependent on the value of $t$).
Notice that this notation requires type checking for the declarations because types now have variables.
In this example, it is necessary to show that variable $t$ is of type $\con^T$.

Finally, the function $\Co{soundOfKind}$ can be declared in the following way:
\begin{equation*}
    \Co{soundOfKind} : k \rightarrow \Sound[k] \mid k : \Kind
\end{equation*}
Here, notation $\Sound[k]$ expresses the projection of type $\Sound$ to only these predicates that are over type $k$. 
This is essential for forming the connection between the types of domain and the range of the function.  
Informally, this declaration aims to express that function $\Co{soundOfKind}$ maps kinds $\Kind$ to the sounds of that kind $\Sound[k]$. 
In particular, based on the type information, sounds of kind $\Sound[k]$ can be any predicate with the typing signature $k \rightarrow \bool$.
Using this information, it is possible to conclude that formula~(\ref{eq:intensional-guarding}) is well-typed.
In particular, given that variable $a$ is of type $\Animal$ and $k$ of type $\Kind$ the following reasoning can be employed to make a judgment $(\$(k)(a) \land \$(\Co{soundOfKind}(k))(a)) : \bool$ (which is the challenging part of formula~(\ref{eq:intensional-guarding})):
\begin{itemize}
    \item $\$(k)(a)$ is well-typed as $k$ is some type symbol that is subtype of $\Animal$ and $a$ is of type $\Animal$, and per Definition~\ref{sec:guarded-osl} types can appear as predicates. 
    \item Since $k$ is a type-symbol, $\$(k)(a)$ can be used for guarding the other part of the conjunction (similar to the (G-\textsc{c}) rule from Definition~\ref{sec:guarded-osl}).
    \item $\$(\Co{soundOfKind}(k))(a)$ is well-typed because: 
        (\textit{i}) $\Co{soundOfKind(k)}$ is of type $k \rightarrow \bool$ 
        (\textit{ii}) Variable $a$ is of type $k$ thanks to the guard $\$(k)(a)$
        (\textit{iii}) $\$(\Co{soundOfKind}(k))(a)$ is of type $\bool$ as variable $a$ (of type $k$) is applied to some predicate of type $k \rightarrow \bool$.
\end{itemize}

To a certain extent, the typing system illustrated in this section resembles the idea of dependent types \cite[Chapter 6, Section 30.5]{pierce2002types}.
In type theory, a type is considered dependent if its definition relies on a value.
For example, a function that adds a new number to a list takes a number and a list of length $n$ as arguments and returns a list of length $n+1$.
Similarly, the function $\Co{soundOfKind}$ takes the intension of a type (a subtype of $\Animal$) as an argument and returns the intension of a unary predicate over that type.
Due to its extensiveness, formalizing such a typing system for order-sorted intensional logic and investigating its relation to dependent types remains within the scope of future work.


\section{Semantics of the language}
\label{sec:semantics}
The formal model semantics of the logic presented in this paper rely on a combination of order-sorted logic~\cite[Section 4.2]{DBLP:journals/ai/BeierleHPSS92} and intensional logic~\cite[Section 3.2]{DBLP:journals/corr/abs-2202-00898}.
Note that in all our examples, extensions of types and functions ranging over concepts are fixed (i.e., $\Sound$ contains exactly $\Bark$ and $\Meow$). This allows for grounding intensional language constructs and semantically reducing the logic to standard \OSL.
However, this section outlines the semantics of the order-sorted intensional logic.
First, we define the notion of structure, a value assignment to vocabulary symbols.
\begin{definition}
	A structure $\Sd$ over an \OSL vocabulary $\Od$ interprets all symbols $s$ in $\Od$ (denoted as $s^\Sd$) such that:
	\begin{enumerate}[leftmargin=2.5em]
		\item The value of each type symbol $\Td$ in $\OTd$ is a non-empty set $\Td^\Sd$.
		\begin{itemize}
            \item Type $\bool$ (boolean) is always assigned the set of truth values $\bool^\Sd = \{\true, \false\}$
		\item Type $\nat$ (natural numbers) is always assigned the set $\nat^\Sd = \{0,1,2,\dots \}$
			\item Type $\con$ (concepts) is assigned the set $\con^\Sd = \{\simV{s} \mid s \in \OSd \cup \OTd\}$. Here $\simV{s}$ is the atomic object formally representing the concept behind the symbol $s$.
		\end{itemize}		
		\item If type symbol $\Td$ is a direct subtype ($\OSTd$) of $\Td_1$, then $\Td^\Sd \subseteq \Td_1^\Sd$.
		\item Each symbol $s$ in $\OSd$ with type signature $\Omd(s) = \atypesign{T}{n}{\Td}$, is assigned a set $s^\Sd \subseteq \Td_1^\Sd \times \dots \times \Td_n^\Sd \times \Td^\Sd$ such that:
		\begin{itemize}
			\item for each tuple $(d_1,\dots,d_n) \in \Td_1^\Sd \times \dots \times \Td_n^\Sd$ there is an element $e \in \Td^\Sd$ such that $(d_1,\dots,d_n,e_1) \in s^\Sd$.
			\item for all tuples $(d_1,\dots,d_n,e_1), (d_1,\dots,d_n,e_2) \in s^\struct$, it holds that $e_1=e_2$.
            \item If $s$ is a type predicate $\Tp{\Td}$, then $s^\Sd = \{(d,\true) \mid d\in \Td\} \cup \{(d,\false) \mid d \in \univ \setminus \Td\}$.
		\end{itemize}
    If $s$ is a function symbol and $(d_1,\dots,d_n,e)\in s^\struct$, we write that $s^\struct(d_1,\dots,d_n)=e$.
	\end{enumerate}
\end{definition}

A common assumption is that each domain object has an identifier, a symbol that makes it possible to directly refer to that value from the theory.
With the notion of a structure formalized, we proceed with defining the value of an expression in a structure. 

\begin{definition}
	\label{def:semantics}
	Given vocabulary $\Od$, let $\alpha$ be an \OSL expression (over $\Od$), and $\Sd$ a structure interpreting all symbols in $\Od$.
	Further, let, for each free variable $x$ occurring in $\alpha$ as an argument of type $\Td$, structure $\Sd$ assign value $x^\struct \in \Td^\Sd$ (with $\Sd[x : d]$ we denote that structure $\struct$ is extended with assignment of value $d$ to variable $x$).   
	The \textbf{value of $\alpha$ in $\struct$}, denoted as $\Vd{\alpha}$, is defined by induction on the structure of $\alpha$: 
	\begingroup
	\allowdisplaybreaks
	\begin{align*}
		\Vd{x} &= x^\struct \quad\qquad \Vd{\Tr} = \true \quad\qquad \Vd{\Fa} = \false\\
		  \Vd{f(\tau_1,\dots,\tau_n)} &=f^\struct(\Vd{\tau_1}, \dots, \Vd{\tau_n}) \quad\qquad \Vd{p(\tau_1,\dots,\tau_n)} = p^\struct(\Vd{\tau_1}, \dots, \Vd{\tau_n})\\
		\Vd{\neg \phi} &=
		\begin{cases}
			\true, &\text{if } \Vd{\phi} = \false;\\
			\false, &\text{if } \Vd{\phi} = \true;
		\end{cases}\\
		\Vd{\phi_1 \lor \phi_2} &=
		\begin{cases}
			\true, &\text{if } \Vd{\phi_1} = \true \text{ or } \Vd{\phi_2} = \true;\\
			\false, &\text{if } \Vd{\phi_1} = \Vd{\phi_2} = \false;
		\end{cases}\\
		\Vd{\exists x [\Td] : \phi} &=
		\begin{cases}
			\true, &\text{if for some } d \in \Td^\Sd, \val{\phi}{\Sd[x : d]} = \true;\\
			\false, &\text{if for all } d \in \Td^\Sd, \val{\phi}{\Sd[x : d]} = \false;
		\end{cases}\\
		\Vd{`(s)} &= \simV{s}, \quad  \text{for } s \in \OSd \cup \OTd\\
		\Vd{\$(\tau)(\tau_1,\dots,\tau_n)} &= s^\struct(\Vd{\tau_1},\dots,\Vd{\tau_n}), \text{ for } s \in \OSd \cup \OTd \text{ such that } \Vd{\tau} = \simV{s}
    \end{align*}
	\endgroup
	An \OSL sentence $\psi$ over vocabulary $\Od$ is \textbf{satisfied} in a structure $\Sd$ (over $\Od$), denote as $\Sd \models \psi$, if and only if $\Vd{\psi} = \true$.
\end{definition}

As discussed in Section~\ref{sec:intensional-osl}, not every order-sorted intensional expression is meaningful (as illustrated by the example: $\Bark(\Tom)$). 
Consequently, attempting to define the value of such expressions is not meaningful. These expressions can be excluded by enhancing the typing system as explained in Section~\ref{sec:osl-i-typeing}. 

\section{Related work and discussion}
\label{sec:related-work}

Frame Logic (F-logic), introduced in \cite{DBLP:conf/sigmod/KiferL89}, is a knowledge representation language that combines conceptual modeling with object-oriented and frame-based languages. 
In this language, it is possible to use types as predicates which is sufficient for expressing formulae like (\ref{eq:conjunction-guarding}) and (\ref{eq:implication-guarding}).
Logic programming incorporating polymorphically order-sorted types is investigated in \cite{DBLP:conf/alp/Smolka88}. 
The Flora-2~\cite{DBLP:journals/jacm/KiferLW95} system combines F-logic and HiLog~\cite{DBLP:journals/jlp/ChenKW93}, resulting in an even more expressive language.
The key differences between these languages and guarded \OSL are:
    \textit{(i)} F-logic is mainly utilizing subtyping from the perspective of object-oriented paradigm while the focus of this paper is on a more general notion of types.
    \textit{(ii)} Results of these papers are related to \emph{parametric polymorphism}~\cite[Chapter 23]{pierce2002types} rather than subtyping polymorphism.
    \emph{Parametric (Ad hoc) polymorphism} includes generic types,  polymorphic  predicate and function symbols 
    and quantification over types.
    An example is $\mother$ and $\father$ functions, mapping animals of a certain kind to another animal of that same kind. Using parametric polymorphism the typing signature of this function can be expressed as ($\mother : \forall k \OSTd \Animal \ .\ k \rightarrow k$). Even though this notation strongly resembles the idea presented in Section~\ref{sec:osl-i-typeing} they are different. Here, variable $k$ ranges over types, while in the other example, this does not have to be the case. However, the dependent type approach with intensional logic can sometimes simulate parametric polymorphism. In this particular example: ($\mother : k \rightarrow k \mid k : \con^T[\Animal]$).
    \textit{(iii)} These languages lack intensional aspects. While HiLog allows for higher-order language constructs, it does not include concepts. This means that using functions such as $\Co{soundOfKind}$ to ``compose'' formulae is not possible. 
    In other words, one can see the intensional logic presented in this paper as a mechanism for expressing templates of formulae. This is because objects from the vocabulary are first-class citizens. This is not the case with the higher-order logic.
    \textit{(iv)} Implicit type guarding is not supported in these languages. 
    In particular, to the best of our knowledge, no other languages use such language constructs (except for our previous work~\cite{DBLP:conf/jelia/MarkovicBD23} where guards ensure the safe application of partial functions).
    However, this paper demonstrates the importance of implicit guarding and power coming from combining it with intensional logic.

The points \textit{(iii)} and \textit{(iv)} suggest that these languages may encounter similar problems to those concerning \OSL discussed in Section~\ref{sec:log-anaysis} and intensional logic from Section~\ref{sec:intensional-osl}.

On the other side, the scope of this paper is limited to subtyping polymorphism. 
Future research should explore how the approach presented in this paper relates to \emph{parametric (Ad hoc) polymorphism}.
In particular, it is worth investigating whether the two typing systems have the same expressive power.
Another research question that opens here is what if we perceive typed logic as a logic of partial predicates, what is then the relation between guarding presented in this paper and guarding that ensures arguments of a function are in its domain of definedness (our previous work~\cite{DBLP:conf/jelia/MarkovicBD23}).

Similar to the approach demonstrated in formula (\ref{eq:intensional-guarding}), it is possible to define higher-order functions in HiLog to map propositions to propositions, thereby achieving similar outcomes. 
However, this approach carries the same issues as the one with intensional logic.
Namely, it requires introduction of new functions and predicates representing the typing relation between different concepts which is redundant as this information is present in the typing signature of these concepts. 
This issue was discussed in Section~\ref{sec:results}.  
Similar issues apply to many imperative programming languages, such as Python, where \emph{dynamic function invocation} can yield similar results but with the price of introducing redundant type information. 
Dynamic function invocation allows one to store names of functions in variables and then invoke these functions by using the variable.

In conclusion, many declarative (logic-based) and imperative languages can achieve similar results as presented in this paper. However, mainly due to the lack of implicit guarding and intensional aspects of the language, these languages do not support the subtyping discussed in this paper as a native language construct. To the best of our knowledge, there are no such knowledge representation languages.

\section{Conclusion}
\label{sec:conclusion}
In this paper, we addressed the challenge of subtyping polymorphism within order-sorted logic. Through our investigation, we identified two essential requirements: intensional logic and implicit guarding with typing assertions. Consequently, we introduced \emph{guarded order-sorted intensional logic} and demonstrated its effectiveness for this task.

The main contributions of this paper are: \textit{(i)} implicit guarding, language constructs introduced in Definition~\ref{def:implicit-guards} allowing conditioning of types for terms based on their application; \textit{(ii)} combining implicit guarding and intensional logic (i.e., quantification over concepts) for expressing subtyping polymorphism, as elaborated in Propositions \ref{prop:intensional-guarding} and \ref{prop:formula-length}.
Additionally, this paper opens two new research topics: the well-typedness conditions of \emph{guarded order-sorted intensional logic} and its relation to dependent types (see Section~\ref{sec:osl-i-typeing}), and second, the relation of order-sorted logic as presented in this work and logic of partial functions (see Section~\ref{sec:related-work}).

\section*{Acknowledgments}
Special thanks to Maurice Bruynooghe for his thorough reviews of this paper. 
Thanks to Robbe Van den Eede and Linde Vanbesien for valuable discussions. 
Thanks to Tobias Reinhard and Justus Fasse for their insightful reviews of the early versions of this paper.

\nocite{*}
\bibliographystyle{eptcs}
\bibliography{biblio}

\begin{thebibliography}{10}
\providecommand{\bibitemdeclare}[2]{}
\providecommand{\surnamestart}{}
\providecommand{\surnameend}{}
\providecommand{\urlprefix}{Available at }
\providecommand{\url}[1]{\texttt{#1}}
\providecommand{\href}[2]{\texttt{#2}}
\providecommand{\urlalt}[2]{\href{#1}{#2}}
\providecommand{\doi}[1]{doi:\urlalt{https://doi.org/#1}{#1}}
\providecommand{\eprint}[1]{arXiv:\urlalt{https://arxiv.org/abs/#1}{#1}}
\providecommand{\bibinfo}[2]{#2}

\bibitemdeclare{article}{DBLP:journals/ai/BeierleHPSS92}
\bibitem{DBLP:journals/ai/BeierleHPSS92}
\bibinfo{author}{Christoph \surnamestart Beierle\surnameend},
  \bibinfo{author}{Ulrich \surnamestart Hedtst{\"{u}}ck\surnameend},
  \bibinfo{author}{Udo \surnamestart Pletat\surnameend},
  \bibinfo{author}{Peter~H. \surnamestart Schmitt\surnameend} \&
  \bibinfo{author}{J{\"{o}}rg~H. \surnamestart Siekmann\surnameend}
  (\bibinfo{year}{1992}): \emph{\bibinfo{title}{An Order-Sorted Logic for
  Knowledge Representation Systems}}.
\newblock {\slshape \bibinfo{journal}{Artif. Intell.}}
  \bibinfo{volume}{55}(\bibinfo{number}{2}), pp. \bibinfo{pages}{149--191},
  \doi{10.1016/0004-3702(92)90055-3}.

\bibitemdeclare{inproceedings}{DBLP:journals/corr/abs-2202-00898}
\bibitem{DBLP:journals/corr/abs-2202-00898}
\bibinfo{author}{Pierre \surnamestart Carbonnelle\surnameend},
  \bibinfo{author}{Matthias \surnamestart van~der Hallen\surnameend} \&
  \bibinfo{author}{Marc \surnamestart Denecker\surnameend}
  (\bibinfo{year}{2023}): \emph{\bibinfo{title}{Quantification and aggregation
  over concepts of the ontology}}.
\newblock In \bibinfo{editor}{Enrico \surnamestart Pontelli\surnameend},
  \bibinfo{editor}{Stefania \surnamestart Costantini\surnameend},
  \bibinfo{editor}{Carmine \surnamestart Dodaro\surnameend},
  \bibinfo{editor}{Sarah~Alice \surnamestart Gaggl\surnameend},
  \bibinfo{editor}{Roberta \surnamestart Calegari\surnameend},
  \bibinfo{editor}{Artur~S. \surnamestart d'Avila Garcez\surnameend},
  \bibinfo{editor}{Francesco \surnamestart Fabiano\surnameend},
  \bibinfo{editor}{Alessandra \surnamestart Mileo\surnameend},
  \bibinfo{editor}{Alessandra \surnamestart Russo\surnameend} \&
  \bibinfo{editor}{Francesca \surnamestart Toni\surnameend}, editors: {\slshape
  \bibinfo{booktitle}{Proceedings 39th International Conference on Logic
  Programming, {ICLP} 2023, Imperial College London, UK, 9th July 2023 - 15th
  July 2023}}, {\slshape \bibinfo{series}{{EPTCS}}} \bibinfo{volume}{385}, pp.
  \bibinfo{pages}{213--226}, \doi{10.4204/EPTCS.385.22}.

\bibitemdeclare{article}{DBLP:journals/jlp/ChenKW93}
\bibitem{DBLP:journals/jlp/ChenKW93}
\bibinfo{author}{Weidong \surnamestart Chen\surnameend},
  \bibinfo{author}{Michael \surnamestart Kifer\surnameend} \&
  \bibinfo{author}{David~Scott \surnamestart Warren\surnameend}
  (\bibinfo{year}{1993}): \emph{\bibinfo{title}{{HILOG:} {A} Foundation for
  Higher-Order Logic Programming}}.
\newblock {\slshape \bibinfo{journal}{J. Log. Program.}}
  \bibinfo{volume}{15}(\bibinfo{number}{3}), pp. \bibinfo{pages}{187--230},
  \doi{10.1016/0743-1066(93)90039-J}.

\bibitemdeclare{article}{DBLP:journals/apal/Fitting04}
\bibitem{DBLP:journals/apal/Fitting04}
\bibinfo{author}{Melvin \surnamestart Fitting\surnameend}
  (\bibinfo{year}{2004}): \emph{\bibinfo{title}{First-order intensional
  logic}}.
\newblock {\slshape \bibinfo{journal}{Ann. Pure Appl. Log.}}
  \bibinfo{volume}{127}(\bibinfo{number}{1-3}), pp. \bibinfo{pages}{171--193},
  \doi{10.1016/J.APAL.2003.11.014}.

\bibitemdeclare{inproceedings}{DBLP:conf/sigmod/KiferL89}
\bibitem{DBLP:conf/sigmod/KiferL89}
\bibinfo{author}{Michael \surnamestart Kifer\surnameend} \&
  \bibinfo{author}{Georg \surnamestart Lausen\surnameend}
  (\bibinfo{year}{1989}): \emph{\bibinfo{title}{F-Logic: {A} Higher-Order
  language for Reasoning about Objects, Inheritance, and Scheme}}.
\newblock In \bibinfo{editor}{James \surnamestart Clifford\surnameend},
  \bibinfo{editor}{Bruce~G. \surnamestart Lindsay\surnameend} \&
  \bibinfo{editor}{David \surnamestart Maier\surnameend}, editors: {\slshape
  \bibinfo{booktitle}{Proceedings of the 1989 {ACM} {SIGMOD} International
  Conference on Management of Data, Portland, Oregon, USA, May 31 - June 2,
  1989}}, \bibinfo{publisher}{{ACM} Press}, pp. \bibinfo{pages}{134--146},
  \doi{10.1145/67544.66939}.

\bibitemdeclare{article}{DBLP:journals/jacm/KiferLW95}
\bibitem{DBLP:journals/jacm/KiferLW95}
\bibinfo{author}{Michael \surnamestart Kifer\surnameend},
  \bibinfo{author}{Georg \surnamestart Lausen\surnameend} \&
  \bibinfo{author}{James \surnamestart Wu\surnameend} (\bibinfo{year}{1995}):
  \emph{\bibinfo{title}{Logical Foundations of Object-Oriented and Frame-Based
  Languages}}.
\newblock {\slshape \bibinfo{journal}{J. {ACM}}}
  \bibinfo{volume}{42}(\bibinfo{number}{4}), pp. \bibinfo{pages}{741--843},
  \doi{10.1145/210332.210335}.

\bibitemdeclare{inproceedings}{DBLP:conf/jelia/MarkovicBD23}
\bibitem{DBLP:conf/jelia/MarkovicBD23}
\bibinfo{author}{Djordje \surnamestart Markovic\surnameend},
  \bibinfo{author}{Maurice \surnamestart Bruynooghe\surnameend} \&
  \bibinfo{author}{Marc \surnamestart Denecker\surnameend}
  (\bibinfo{year}{2023}): \emph{\bibinfo{title}{Towards Systematic Treatment of
  Partial Functions in Knowledge Representation}}.
\newblock In \bibinfo{editor}{Sarah~Alice \surnamestart Gaggl\surnameend},
  \bibinfo{editor}{Maria~Vanina \surnamestart Martinez\surnameend} \&
  \bibinfo{editor}{Magdalena \surnamestart Ortiz\surnameend}, editors:
  {\slshape \bibinfo{booktitle}{Logics in Artificial Intelligence - 18th
  European Conference, {JELIA} 2023, Dresden, Germany, September 20-22, 2023,
  Proceedings}}, {\slshape \bibinfo{series}{Lecture Notes in Computer Science}}
  \bibinfo{volume}{14281}, \bibinfo{publisher}{Springer}, pp.
  \bibinfo{pages}{756--770}, \doi{10.1007/978-3-031-43619-2\_51}.

\bibitemdeclare{book}{pierce2002types}
\bibitem{pierce2002types}
\bibinfo{author}{Benjamin~C \surnamestart Pierce\surnameend}
  (\bibinfo{year}{2002}): \emph{\bibinfo{title}{Types and programming
  languages}}.
\newblock \bibinfo{publisher}{MIT press}.

\bibitemdeclare{inproceedings}{DBLP:conf/alp/Smolka88}
\bibitem{DBLP:conf/alp/Smolka88}
\bibinfo{author}{Gert \surnamestart Smolka\surnameend} (\bibinfo{year}{1988}):
  \emph{\bibinfo{title}{Logic Programming with Polymorphically Order-Sorted
  Types}}.
\newblock In \bibinfo{editor}{Jan \surnamestart Grabowski\surnameend},
  \bibinfo{editor}{Pierre \surnamestart Lescanne\surnameend} \&
  \bibinfo{editor}{Wolfgang \surnamestart Wechler\surnameend}, editors:
  {\slshape \bibinfo{booktitle}{Algebraic and Logic Programming, International
  Workshop, Gaussig, GDR, November 14-18, 1988, Proceedings}}, {\slshape
  \bibinfo{series}{Lecture Notes in Computer Science}} \bibinfo{volume}{343},
  \bibinfo{publisher}{Springer}, pp. \bibinfo{pages}{53--70},
  \doi{10.1007/3-540-50667-5\_58}.

\bibitemdeclare{article}{DBLP:journals/jsyml/Wang52}
\bibitem{DBLP:journals/jsyml/Wang52}
\bibinfo{author}{Hao \surnamestart Wang\surnameend} (\bibinfo{year}{1952}):
  \emph{\bibinfo{title}{Logic of Many-Sorted Theories}}.
\newblock {\slshape \bibinfo{journal}{J. Symb. Log.}}
  \bibinfo{volume}{17}(\bibinfo{number}{2}), pp. \bibinfo{pages}{105--116},
  \doi{10.2307/2266241}.

\end{thebibliography}
\end{document}